\documentclass[conference]{IEEEtran}
\IEEEoverridecommandlockouts

\usepackage{cite}
\usepackage{amsmath,amssymb,amsfonts}
\usepackage{graphicx}
\graphicspath{{figures/}}
\usepackage{textcomp}
\usepackage{xcolor}
\usepackage{booktabs}
\usepackage{array}
\usepackage{tabularx}
\usepackage{multirow}
\usepackage{url}
\usepackage{hyperref}

\newcommand{\COFR}{\mathrm{COFR}}
\usepackage{amsthm}
\newtheorem{proposition}{Proposition}

\begin{document}
\bstctlcite{BSTcontrol}  

\title{Accuracy Is Not Service: A Decision-Aware Benchmark for Intermittent-Demand Forecasting
\thanks{Shih-Fen Cheng would like to acknowledge the funding support from the Singapore Ministry of Education (MOE) Academic Research Fund (AcRF) Tier 1 Grant (Project Number: 2026-UI-002).}
}

\author{\IEEEauthorblockN{Joo Ern Chin\textsuperscript{1,2}, Shih-Fen Cheng\textsuperscript{1}, and Aldy Gunawan\textsuperscript{1}}
\IEEEauthorblockA{\textsuperscript{1}\textit{School of Computing and Information Systems, Singapore Management University}\\
80 Stamford Road, Singapore 178902\\
jooern.chin.2025@engd.smu.edu.sg \quad \{sfcheng, aldygunawan\}@smu.edu.sg}
\IEEEauthorblockA{\textsuperscript{2}\textit{ST Logistics Pte. Ltd.}\\
5 Clementi Loop, Singapore 129816}
}

\maketitle
\raggedbottom 

\begin{abstract}

A contract-logistics spare-parts operator is paid on order-level service: an order
counts only if \emph{every} requested line is fulfilled, yet forecasters are
selected based on line-level forecast accuracy. This disconnect matters when
demand is intermittent and lumpy, histories are short, and lead times span months. We benchmarked 38 forecasting methods spanning classical, intermittent-demand, machine-learning, deep-learning, and pretrained foundation models. A common decision-aware protocol evaluates them on an industrial panel drawn from a live contract and two public datasets. Forecast-accuracy rank and order-service rank are \emph{negatively} correlated
on the industrial panel, at -0.555, across methods evaluated on
20{,}330 real multi-item orders.
Service is more closely associated with the \emph{direction} of cumulative forecast bias,
including over-prediction during zero-demand periods, than with point accuracy. Examining bias in
Chronos-2's instance normalization yields a training-free correction that
lifts the per-material fill proxy from 77.5\% to 92.0\% (14.5 percentage points) at the 90\% policy target
and raises the complete-order fill rate from 54\% to 63\%. For reproducibility, we release \emph{RUF} (Regenerate-Until-Fidelity), a
method for generating fidelity-certified synthetic panels on which the findings
reproduce. For intermittent demand, the lowest-error forecast need
not deliver the highest service. Bias direction helps explain
this gap, which can be reduced without retraining.

\end{abstract}

\begin{IEEEkeywords}
intermittent demand, spare parts, forecasting benchmark, foundation models, decision-aware evaluation, applied machine learning
\end{IEEEkeywords}

\section{Introduction}
\label{sec:intro}
Keeping fielded equipment operational requires spare parts that may be ordered
only a few times a year, making inventory planning difficult. A contract-logistics
spare-parts operator is judged on order-level service and penalized for shortfalls; as
an operating threshold, we adopt an illustrative \emph{80-of-80} target
(80\% of customer orders fully filled across 80\% of platforms,
the equipment fleets a contract covers). Demand histories are sparse and replenishment lead times are measured in months.
Service requirements are strict: a multi-item order is incomplete if any component is short. In this setting,
planners use simpler, auditable rules. Like much of the spare-parts and maintenance, repair, and operations (MRO) industry, the operator we study drives
replenishment with moving averages in a spreadsheet or enterprise resource planning (ERP) system
rather than a learned model.

The operational question is which forecaster should inform replenishment.
Accuracy-based selection assumes that the most accurate model is also the best choice. This assumption fails on our industrial panel: moving averages perform well on accuracy, but their bias direction can
reduce service under the common replenishment policy.

Two gaps limit evidence-based method choice here. First, broad
\emph{cross-class} comparison on genuinely intermittent industrial demand is
scarce. Intermittent-demand studies typically compare specialized
estimators, and machine-learning studies compare learners. Recent
foundation-model evaluations often emphasize structurally smoother retail or
macroeconomic series \cite{makridakis_m5_2022,petropoulos_forecasting_2022}.
Second, method choice still defaults to point accuracy. The conventional pipeline
illustrates this assumption: a forecaster is selected on accuracy (mean absolute scaled
error, MASE; symmetric mean absolute percentage error, sMAPE),
its point forecast sets the order-up-to level through the textbook base-stock
formula, and lower error is presumed to translate into higher service
\cite{boylan_syntetos_2021}. Competitions rank methods on accuracy
\cite{makridakis_m5_2022}, and intermittent-demand research tunes bias corrections
to accuracy targets \cite{syntetos2005accuracy}. The forecast--inventory gap is
acknowledged \cite{goltsos2022mindgap}, but it is rarely measured \emph{across
model classes} against an operational target. Accuracy rankings therefore need
validation against service outcomes.

We address these gaps by benchmarking 38 methods under a common decision-aware
protocol and examining how forecast bias
direction contributes to the resulting accuracy--service inversion. We show that the associated bias can be reduced within a foundation
model without retraining and reproduce the findings on a fidelity-certified
synthetic panel.

Our contributions are as follows:
\begin{enumerate}
  \item \textbf{RUF (Regenerate-Until-Fidelity):} we release a fidelity-certified public
  benchmark (CC-BY-4.0, with full evaluation code and open data in a
  public reproduction repository\footnote{\url{https://github.com/sfcheng-research/icdm-2026-reproduction}\label{fn:repo}}), certified against the
  confidential data at the level of fitted marginal distributions and cross-material
  copula dependence, on which the leaderboard,
  inversion (directionally), and correction reproduce. The procedure could support certified benchmarks in other domains.

  \item \textbf{A 38-method cross-class benchmark} for \emph{decision-aware}
  intermittent spare-parts forecasting, spanning eight model classes, three sparsity
  regimes (industrial, RAF, M5), and 31 demand-pattern segments, run under an
  \textbf{audited, no-fallback protocol} (every method regenerated from source,
  failures logged and counted, with no silent baseline substitution),
  which we propose for benchmark releases.

  \item \textbf{A bias-direction diagnostic
  (BDD)}: on intermittent demand the accuracy ranking can \emph{reverse} against the
  contractual service target: across the $n=38$ methods, each scored on 20{,}330 real
  multi-item orders, the Spearman correlation is $\rho=-0.555$ (95\% confidence interval (CI)
  $[-0.78,-0.26]$), strengthening to $\rho=-0.88$ once near-duplicate variants are
  collapsed to one per model class. We examine forecast bias direction through
  its effect on stock targets (Proposition~\ref{prop:bias}), a scalar
  $\alpha$-sweep, and a \emph{hurdle} model that forecasts demand
  \emph{occurrence} and \emph{size} separately. Realized service also depends on
  replenishment timing and order allocation.

  \item \textbf{A zero-training-cost correction} to Chronos-2's instance normalization that
  increases the per-material fill proxy by 14.5 percentage points on the industrial panel
  and 17.1 percentage points on RAF at the 90\% policy target, and raises industrial
  complete-order fill rate from 54\% to 63\%, illustrating the diagnostic's use.
\end{enumerate}

\section{Related Work}
\label{sec:background}
Work on intermittent-demand forecasting includes Croston's
occurrence--size decomposition \cite{croston_forecasting_1972} and its bias
corrections SBA \cite{syntetos2005accuracy} and TSB \cite{teunter_intermittent_2011},
temporal-aggregation heuristics \cite{nikolopoulos_aggregatedisaggregate_2011},
and the Syntetos--Boylan demand taxonomy \cite{boylan2008classification}. The M5 competition evaluated feature-based gradient boosting and global deep models \cite{makridakis_m5_2022,theodorou_forecast_2025}, and
pretrained \emph{foundation} time-series models (Chronos
\cite{ansari_chronos_2024}, TimesFM \cite{das_timesfm_2024}, and Moirai
\cite{woo_moirai_2024}) report competitive zero-shot performance. Other studies forecast the lead-time demand \emph{distribution} directly
by resampling the empirical demand history rather than estimating a point mean (Willemain's
Markov-chain bootstrap \cite{willemain_new_2004} and its empirical variants
\cite{porras2008inventory}). Across these forecasting branches, evaluation is constrained by
data. The public anchors (M4/M5 \cite{makridakis_m4_2018,makridakis_m5_2022}) and the
general time-series benchmarks now used to rank foundation models (the Monash archive
\cite{godahewa_monash_2021} and GIFT-Eval \cite{aksu_gift_2024}) are dominated by
smooth, regularly sampled series with little zero-inflated spare-parts demand.
Realistic intermittent panels are also proprietary. This limits independent
replication of decision-facing findings on sparse demand. RUF addresses this
gap (Section~\ref{sec:methodology}).

Transportation studies also connect demand prediction with operational data
and outcomes. M$^2$--CNN combines spatiotemporal demand patterns with
microscopic information from vacant-taxi movements \cite{cheng2023m2},
while Agarwal et al.\ show that surge-informed demand predictions improve
taxi operations in a driver-guidance analysis \cite{agarwal_trc_2022}.
Our focus is how forecast-accuracy rankings translate into complete-order
service rankings under heavy intermittency.

End-to-end methods train predictors against decision cost \cite{elmachtoub_grigas_2022}
or optimize stock directly from data \cite{qi_endtoend_2023}. Their advantage is
problem-dependent \cite{cameron_perils_2022}; we retain two stages because operators
need explainable, auditable policies that can be retuned per platform.

The forecast--inventory interaction, though acknowledged
\cite{petropoulos_forecasting_2022,goltsos2022mindgap}, is seldom measured on one panel
across many classes against an operational target.
M5-style studies primarily rank methods on accuracy
\cite{makridakis_m5_2022}, and the closest prior work measures a forecast--inventory
link on smooth M5 retail data and finds it weakly \emph{positive}
\cite{theodorou_forecast_2025}. We evaluate 38 methods across model classes against a
contractual \emph{complete-order} service target under heavy intermittency and find
that the relationship \emph{reverses}. Classical intermittent-demand work
calibrates bias corrections within the Croston family precisely to \emph{remove}
over-forecast bias and improve accuracy
\cite{syntetos2005accuracy,teunter_intermittent_2011,syntetos_bias_2001}. Our
finding is that under a complete-order contract, the same bias
can improve service. Removing it for accuracy can reduce service.

\section{Methodology}
\label{sec:methodology}
\label{sec:data}
\subsection{Panels}
The industrial panel contains 4{,}021 benchmark-eligible
materials  (the operator's term for a stocked part, or stock-keeping unit (SKU))\footnote{Panel
slices: 4{,}021 for accuracy (rolling-origin); 3{,}828 of these scored
for order service (20{,}330 orders, 41 platforms); 4{,}558 for pooled
accuracy and service in the Chronos study, including the 537 materials outside
the main benchmark. RUF is a separate synthetic
corpus (5{,}000 materials).} over 81
monthly periods (69 for training, 12 for testing), with limited covariates, long lead times, and
multi-item orders. In the full characterization panel (Table~\ref{tab:data}),
70\% of monthly cells are zero, the average demand interval (ADI) is 4.04, and
89.7\% of series are intermittent or lumpy. The panel is drawn from a
contract-logistics operator's live spare-parts contract; material identifiers and
prices are anonymized, and the panel itself cannot be released under the
operator's commercial terms. This constraint motivates the public RUF
surrogate of Section~\ref{sec:benchmark-release}. We replicate the benchmark on the
public RAF spare-parts panel (5{,}000 series, 90\% zeros; benchmarked on a
stratified 500-series sample, Section~\ref{sec:robustness})
\cite{eaves_kingsman_2004} and use a 250-item monthly M5 panel,
stratified toward M5's sparsest items but still substantially less intermittent than
spare-parts demand, as a low-intermittency anchor (Table~\ref{tab:data}). This contrast tests how rankings vary by regime.

\begin{table}[t]
\centering\footnotesize
\caption{Monthly demand profiles over the full characterization windows:
industrial (4,558 materials, 81 months), RAF (5,000, 84), and M5 (250, 60).
ADI: average demand interval; CV$^2$: squared coefficient of variation.}
\label{tab:data}
\begin{tabular}{lcccc}
\toprule
Dataset & Zero frac. & ADI & CV$^2$ & Interm.+Lumpy \% \\
\midrule
Industrial panel & 0.70 & 4.04 & 2.46 & 89.7 \\
RAF (public)      & 0.90 & 8.65 & 0.63 & 100.0 \\
M5 (stratified) & 0.33 & 1.81 & 0.22 & 56.8 \\
\bottomrule
\end{tabular}
\end{table}

\subsection{RUF: Regenerate-Until-Fidelity generation}
\label{sec:benchmark-release}
We release a certified, 5{,}000-material \textbf{RUF} surrogate of the confidential
panel (seed 42; datasheet in the release; CC-BY-4.0; see footnote~\ref{fn:repo}).
Its generator has three stages.
\emph{(1)~Template sampling:} each material draws a real SKU's aggregate
template: its best-fit distribution family (Poisson, negative-binomial, zero-inflated,
hurdle, compound-Poisson, or gamma-mixture), that family's fitted parameters, a
target zero fraction and CV$^2$, and a platform. \emph{(2)~Occurrence$\times$size synthesis:} demand is modeled as an occurrence process multiplied by a
conditional demand size. A platform-level Gaussian factor couples monthly draws.
Well-determined fits use inverse cumulative distribution functions (CDFs);
otherwise Gaussian-threshold occurrence draws are combined with moment-matched
sizes, including continuous gamma mixtures rounded to positive counts.
Poisson templates use a native count branch. Occurrence probabilities and demand
sizes vary with the template's first-/second-half demand ratio; a per-SKU
empirical percentile bounds the tail. This mechanism models co-occurrence and
non-stationarity rather than imposing log-normal renewal intervals.
Multi-item orders and right-skewed moving-average-price (MAP) unit costs are synthesized, and $\sim$20\%
of materials are designated as cold-start items. \emph{(3)~Regenerate-until-fidelity:} each
material is \emph{redrawn until its realized Syntetos--Boylan class and zero
fraction match its real template} (up to ten draws, retaining the best match and scoring CV$^2$;
84\% of non-cold-start materials pass, and the release includes the
per-material audit), adding
a sample-level fidelity check to aggregate comparisons.

We assess fidelity beyond summary moments at two levels
(Table~\ref{tab:fidelity}; matched 800-SKU and $\sim$2{,}900-pair stratified samples,
fitted using the same code path on a shared 60-month window). \emph{Marginals:}
the distribution of best-fitting families across materials is similar (total-variation
distance (TVD) 0.038, the same most common family (a gamma mixture) on both panels; zero fraction (median) 0.750
vs.\ 0.729, dispersion
index 2.06 vs.\ 2.29). \emph{Dependence:} a bivariate-copula screen matches upper-tail
dependence within 0.01 ($\lambda_U$ 0.107 vs.\ 0.116) and copula-family mix
(TVD 0.068);
both panels exhibit weak tail dependence. The panel approximates the real
data's marginal-family distribution \emph{and} cross-material dependence. The leaderboard, inversion direction, and correction reproduce
on it (Section~\ref{sec:repro}).

\begin{table}[t]
\centering\footnotesize
\caption{Synthetic versus real data: fidelity on matched samples.}
\label{tab:fidelity}
\setlength{\tabcolsep}{5pt}
\begin{tabular}{lrr}
\toprule
Fidelity check & Real & Synthetic \\
\midrule
Top marginal family (gamma-mix2) share & 0.445 & 0.446 \\
Winner-family mix (TVD, lower=closer)   & \multicolumn{2}{c}{0.038} \\
Zero fraction (median)            & 0.750 & 0.729 \\
Dispersion index (median)               & 2.06 & 2.29 \\
Mean upper-tail dependence $\lambda_U$  & 0.107 & 0.116 \\
Mean lower-tail dependence $\lambda_L$  & 0.029 & 0.018 \\
Copula-family mix (TVD)                 & \multicolumn{2}{c}{0.068} \\
\bottomrule
\end{tabular}
\end{table}

\subsection{Demand segmentation}
\label{sec:segments}
The segmented benchmark contains intermittent/lumpy materials; smooth/erratic
materials in Table~\ref{tab:data}'s broader characterization panel are excluded.
The Syntetos--Boylan class (intermittent or lumpy) is crossed with structural
traits, including level shifts, autocorrelation, trend (up or down), seasonality,
ultra-sparsity, and high dispersion. The \textbf{31 demand-pattern segments}
with $\ge10$ materials cover 3,914 of the 4,021 eligible materials; 107 in
smaller groups are omitted from this segment analysis. Boundaries are ADI
$1.32$, CV$^2$ $0.49$, and sparsity tiers at zero-fraction $0.30$ and $0.80$.
These segments define the blocks for the Friedman test and identify
where each model class performs well (Section~\ref{sec:benchmark}).

\noindent\textbf{The 38-method roster.}
The benchmark spans \textbf{38 methods across eight model classes}: 7 classical
(Naive, SeasonalNaive, SES, Holt, ETS, ARIMA, Prophet~\cite{taylor_letham_2018}), 5 intermittent specialists
(Croston, SBA, TSB, ADIDA, Bootstrap), 3 simple averages (MA, WMA, EMA), 9
feature-based machine-learning (ML) methods (LightGBM~\cite{ke_lightgbm_2017}, LightGBM-Ensemble, LightGBM-Quantile, XGBoost~\cite{chen_xgboost_2016}, RF,
GradientBoosting, GlobalMLP, SVR, ElasticNet), 7 deep sequence
models (DeepAR~\cite{salinas_deepar_2020}, DeepAR-Tweedie, N-BEATS~\cite{oreshkin_n-beats_2020}, TFT~\cite{lim_tft_2021}, Transformer, RNN, TwoStageML), 2 state-space models (TBATS,
BATS)~\cite{delivera_tbats_2011}, 1 hybrid (HybridHierarchical), and 4 pretrained foundation models
(Chronos-Bolt tiny/small, TimesFM, Moirai). For this cross-class comparison, we distinguish two Chronos model families:
the benchmark uses \emph{Chronos-Bolt}~\cite{ansari_chronos_2024,aws_chronosbolt_2024}, the distilled fast-inference model, while
the adaptation of Section~\ref{sec:adaptation} uses \emph{Chronos-2}~\cite{ansari_chronos2_2025}, whose
instance normalization is exposed for the modification evaluated there.

\noindent\textbf{Protocol, metrics, and audit.}
Using forecast origins \cite{bergmeir_benitez_2012}, each method $m$ turns
history into a demand-rate forecast $\hat r_{i,o,m}$ for material $i$; we evaluate the
demand it predicts over a full replenishment lead time,
$\hat Y^{(H)}_{i,o}=H\hat r_{i,o,m}$, against observed demand
(Eq.~\ref{eq:ltd}), to match the replenishment horizon.
\begin{equation}\label{eq:ltd}
Y^{(H)}_{i,o}=\sum_{h=1}^{H} y_{i,o+h},\qquad \hat Y^{(H)}_{i,o}=H\,\hat r_{i,o,m}.
\end{equation}
Table~\ref{tab:leaderboard} and the industrial order-book accuracy ranks use a
single six-month lead-time window: origin 69 for industrial and 78 for RAF.
The reference predicts $H\bar y_{i,o}^{(3)}$, where $\bar y_{i,o}^{(3)}$ is the
last three training observations' mean; the method named Naive instead repeats
the last observation. Set $e_i=|Y^{(H)}_{i,o}-\hat Y^{(H)}_{i,o,m}|$ and
$d_i=|Y^{(H)}_{i,o}-H\bar y_{i,o}^{(3)}|$. The benchmark convention is
\begin{equation}\label{eq:mase}
\mathrm{MASE}_m=\frac{1}{|\mathcal{M}_m|}\sum_{i\in\mathcal{M}_m}
\begin{cases}e_i/d_i,&d_i>0,\\e_i,&d_i=0,\end{cases}
\end{equation}
where $\mathcal{M}_m$ contains valid material forecasts. This legacy score uses
unscaled absolute error when the reference is exact; it is not conventional
in-sample-scaled MASE, nor does a value below 1 universally imply superiority
to the reference. Lower is better. Rolling-origin adaptation instead uses the
pooled ratio $\sum_{i,o}e_{i,o}/\sum_{i,o}d_{i,o}$, with an undefined score if the
pooled denominator is zero. We identify the aggregation for each comparison.
Rankings are stable across naive scales (Section~\ref{sec:robustness}).
We propose an \emph{audited, no-fallback protocol} for benchmark releases.
Every method's result stream is regenerated from source. Any failure
(non-convergence, missing forecast) is \emph{logged and counted,
never silently replaced by a baseline}. Such substitution would overstate a
method's coverage. We cite only audited streams.

\noindent\textbf{Non-seasonality.}
A four-detector seasonality audit (classical and locally smoothed seasonal-trend
decomposition, Fourier power-ratio, and lag-12 autocorrelation, all purpose-built
for a period-12 signal) flags at most 5.3\% of the
industrial panel and 2.4\% of RAF as seasonal. These low rates suggest that
seasonal misspecification is unlikely to be the main driver of simple rate estimators' competitiveness
or the accuracy--service comparison.

\subsection{From forecasts to service}
\label{sec:setting}
To translate forecasts into service outcomes, we evaluate every method under one \emph{common}
order-up-to policy (at each review, replenishment is ordered toward a target level $S$), so any
service difference is attributable to the forecast, not
to per-method tuning. The base-stock level uses the standard normal-approximation
form (Eq.~\ref{eq:basestock}), with protection period $L+R$ (lead time plus
review interval) and a single safety
factor $z=\Phi^{-1}(\tau)$ shared across all methods at policy target $\tau$:
\begin{equation}\label{eq:basestock}
S_{i,m}=\hat\mu_{i,m}\,(L+R)+z\,\hat\sigma_{i,m}\sqrt{L+R}.
\end{equation}
Here $L=6$ months and $R=1$ month unless varied; $\hat\mu$ is the mean rate of
the extracted point forecasts and $\hat\sigma$ is the forecaster-supplied
dispersion estimate, whose estimation is method-specific. Targets are clipped
at zero. Each replay starts with $S$ on hand and an empty replenishment pipeline.
Receipts precede demand; orders placed after demand arrive $L$ months later.
Service is the complete-order fill rate (COFR): replaying the operator's actual
multi-item orders, an order succeeds only if \emph{every} requested line is
available within the window \cite{song_order_1998},
\begin{equation}\label{eq:cofr}
\COFR_m=\frac{1}{|\mathcal{Q}|}\sum_{q\in\mathcal{Q}}
\mathbf{1}\!\Big\{\min_{j\in\mathcal{I}_q}F_{qj,m}=1\Big\},
\end{equation}
where $F_{qj,m}=1$ if and only if line $j$ of order $q$ is filled on time.
Orders are processed in a fixed common sequence. An order that cannot be filled
in full counts as a failure, consumes no stock, and is not backordered; successful
orders consume every requested line. After each monthly review, the replay
orders $\max(0,S-\text{inventory position})$ for each material.
Our representative
\emph{80-of-80} target aggregates this to the platform level (at least 80\% of orders
complete for at least 80\% of platforms); it tracks
the order-level metric (Section~\ref{sec:frontier}, Table~\ref{tab:frontier}). \emph{Optimizing} the per-platform axis (reserved allocation,
per-platform safety factors) is a separate inventory-control problem. Because a single short line makes an
order incomplete, the implicit-safety-stock effect of
Section~\ref{sec:diagnostic} can affect complete-order service.

\section{Experiments}
\label{sec:benchmark}
Table~\ref{tab:leaderboard} summarizes industrial and RAF accuracy
(18 of 38 methods shown across all eight classes; full roster in the release).
Four findings emerge.

\begin{table}[t]
\centering\footnotesize
\caption{Per-material lead-time MASE (Eq.~\ref{eq:mase}) and rank:
industrial order-book panel (3,828 materials) and RAF (500).
$\dagger$: no RAF rank (no-fallback subset).}
\label{tab:leaderboard}
\setlength{\tabcolsep}{5pt}
\begin{tabular}{lrrrr}
\toprule
 & \multicolumn{2}{c}{Industrial} & \multicolumn{2}{c}{RAF} \\
\cmidrule(lr){2-3}\cmidrule(lr){4-5}
Method & Rank & MASE & Rank & MASE \\
\midrule
WMA                 & 1  & 0.803 & 7  & 2.011 \\
Chronos-Bolt-tiny   & 2  & 0.890 & 4  & 1.095 \\
Chronos-Bolt-small  & 3  & 0.895 & 3  & 1.057 \\
Holt          & 4  & 0.922 & 10 & 2.931 \\
MA            & 5  & 0.971 & 12 & 3.155 \\
SES           & 6  & 0.995 & 8  & 2.905 \\
TimesFM             & 8  & 1.044 & 2  & 1.048 \\
LightGBM-Quantile & 9  & 1.060 & 1  & 0.943 \\
SVR           & 10 & 1.118 & 5  & 1.143 \\
Naive               & 11 & 1.147 & 6  & 1.169 \\
DeepAR$^\dagger$ & 12 & 1.316 & -- & 1.615 \\
TSB                 & 15 & 1.512 & 15 & 4.102 \\
TBATS         & 18 & 1.702 & 18 & 4.662 \\
XGBoost       & 23 & 2.154 & 20 & 5.090 \\
Moirai              & 34 & 3.321 & 32 & 14.178 \\
Croston             & 35 & 3.382 & 31 & 8.062 \\
HybridHierarchical & 36 & 3.519 & 29 & 7.681 \\
ADIDA         & 37 & 3.548 & 28 & 7.380 \\
\bottomrule
\end{tabular}
\end{table}

First, classical intermittent-demand methods do not lead the industrial
accuracy ranking. On the order-book panel a weighted moving average
(WMA, MASE 0.803) ranks first, with the two Chronos-Bolt foundation models just
behind (0.890, 0.895) and the intermittent specialists at the bottom. A sparsity analysis helps explain this result: native-frequency smoothing beats temporal
aggregation in 70--88\% of materials, and the best moving-average window shrinks as
sparsity rises (MA(4) at 0--20\% zeros, MA(1) above 80\%). (Under a pooled MASE on the
full benchmark the pretrained models lead instead, with TimesFM at 0.831; the ordering is
aggregation-dependent, so we anchor the service analysis to the order book.)

Second, RAF preserves the cross-class picture: the Chronos-Bolt
foundation models, TimesFM, and a quantile gradient-boosting model lead, while Moirai has much higher error
(per-material MASE 14.2). Foundation-model status alone does not imply low error on either panel.
Six out of seven deep-sequence methods receive no RAF rank: five are scored on no-fallback subsets,
and the audit excludes DeepAR-Tweedie's unstable run; TwoStageML ranks.

Third, accuracy and forecast \emph{bias direction} come apart: the low-MASE leaders
have neutral-to-negative cumulative bias, whereas the intermittent
specialists over-predict demand during zero-demand periods. This asymmetry helps
explain the service outcomes (Section~\ref{sec:diagnostic}).

Fourth, aggregate ranks do not establish dominance across demand patterns. The Friedman test ranks methods
\emph{within} each of the 31
demand-pattern segments (Section~\ref{sec:segments}) and averages those ranks to
compare methods \emph{across} demand types rather than by aggregate
error. The Friedman test rejects equality ($\chi^2=239.68$, $p<10^{-30}$, 31
segments), but the Nemenyi critical difference is wide ($\mathrm{CD}=10.90$),
wider than the entire 9.1-rank spread of the top-20 Friedman mean ranks ($11.0$ to
$20.1$, Fig.~\ref{fig:cd}): a single connecting bar would join all of these leading methods.
The segment-blocked TimesFM--WMA comparison is also inconclusive
(unadjusted $t$-test $p=0.110$); this does not establish equivalence.
The Nemenyi comparison does not separate the leading methods from a broad field
including moving averages and the Croston family.
WMA leads on order-book MASE (Table~\ref{tab:leaderboard}) but does not
consistently lead across demand types. Service offers a further deployment criterion.

\begin{figure*}[t]
\centering
\includegraphics[width=\textwidth]{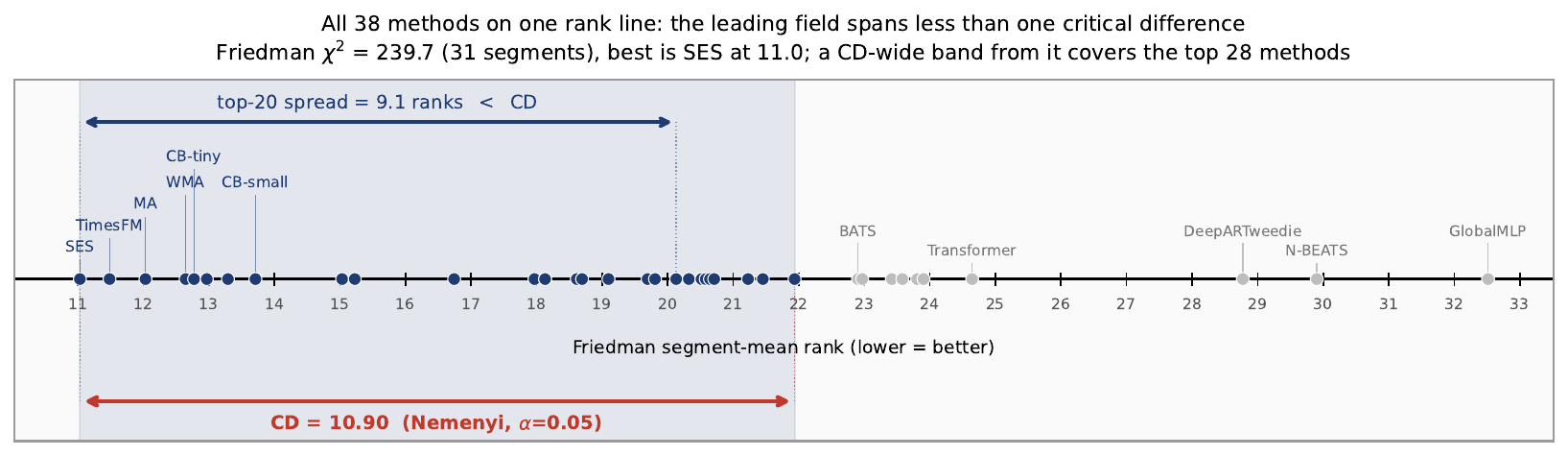}
\caption{All 38 methods on the Friedman segment-mean-rank axis. A
critical-difference-wide band anchored at the best method \emph{on this axis} (SES,
which need not match the order-book MASE leader WMA, as the two use different
aggregations) covers the top 28. Lower-ranked methods fall outside
the band.}
\label{fig:cd}
\end{figure*}

The same picture holds segment by segment (Fig.~\ref{fig:segheat}): ranking all 38
methods \emph{within} each of the 31 demand-pattern segments, no class is best
everywhere. Simple averages, the foundation models, and the lone hybrid method lead
the most segments (classes with one or two methods have high-variance means);
intermittent specialists win specific level-shift patterns and rank poorly on
others; every class performs poorly in some segments. Deployment therefore requires method choice by demand pattern, beyond aggregate ranks.

\begin{figure*}[t]
\centering
\includegraphics[width=\textwidth]{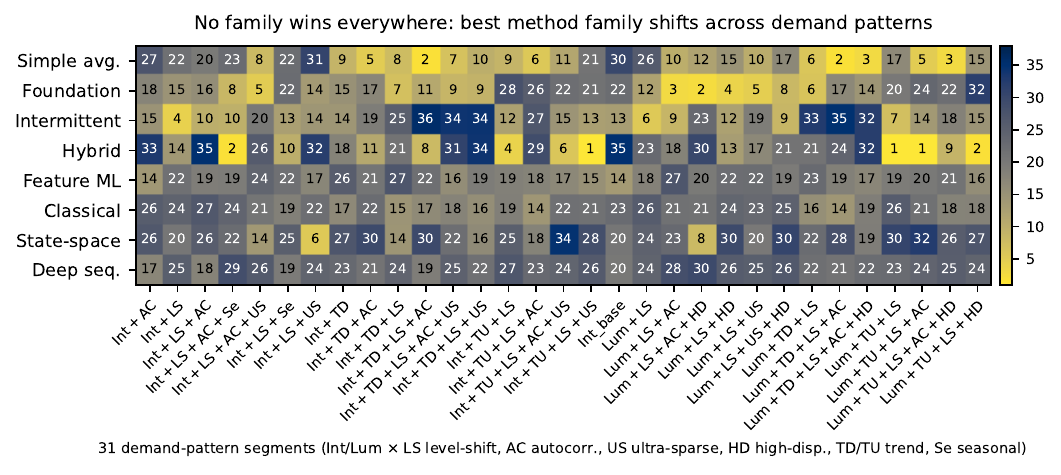}
\caption{Mean within-segment MASE rank (1 = best of 38) per model class across the 31
demand-pattern segments. Bright (yellow) indicates a low mean rank within a segment.}
\label{fig:segheat}
\end{figure*}

\subsection{The bias-direction diagnostic (BDD): when accuracy and service rankings diverge}
\label{sec:diagnostic}
We test whether forecast rankings translate into complete-order service rankings. Replaying the
operator's real multi-item orders across 41 platforms under the common policy of
Eq.~\ref{eq:basestock}, we correlate each method's MASE rank with its
complete-order fill-rate rank (COFR, Eq.~\ref{eq:cofr}). Across the full order book
(20{,}330 orders, 3{,}828 forecastable materials) the
correlation at the 80\% policy target is \emph{negative}: $\rho=-0.555$
($p=3\times10^{-4}$, $n=38$; 95\% bootstrap CI $[-0.78,-0.26]$), and it stays negative
at every service target from 0.70 to 0.95 ($\rho$ from $-0.61$ to $-0.48$). Across methods, greater
accuracy is associated with lower contractual service. Using the complete
order book rules out stratified subsampling as the source of this effect. The 38 methods include near-duplicate variants (three moving averages, two Chronos sizes,
several boosting variants). Collapsing each of the eight model classes to a single
representative provides a sensitivity check, and the inversion \emph{strengthens}
($\rho=-0.88$, $n=8$, $p=0.004$; $\rho=-0.81$ on class means). The association persists under equal class weights.

\noindent\textbf{Why accuracy inverts.}
Forecast means affect service through the inventory targets, subject to the
replay's allocation and replenishment rules.

\begin{proposition}[Forecast mean and the stock target]
\label{prop:bias}
Fix $L+R>0$, $z$, and identical dispersion estimates for two forecasts.
If $\hat\mu_{i,a}\ge\hat\mu_{i,b}$ for every material $i$, then
$S_{i,a}\ge S_{i,b}$ in Eq.~\ref{eq:basestock}. Before clipping, each target
difference equals $(L+R)(\hat\mu_{i,a}-\hat\mu_{i,b})$.
\end{proposition}
\begin{proof}
Subtract the two target equations. Their safety terms cancel; multiplication by
$L+R>0$ and clipping at zero preserve the componentwise ordering.
\end{proof}

This is a target ordering, not a pathwise service guarantee. In the complete-order
replay, rejecting an unfillable order preserves stock for later orders; a larger
target can instead serve that order and deplete stock. Partial fulfillment also
does not ensure pathwise monotonicity when higher targets change replenishment
timing. The service association is empirical, and a larger panel-average mean
alone does not imply the componentwise premise of Proposition~\ref{prop:bias}.

We distinguish signed lead-time bias, the mean of
$\hat Y^{(H)}_{i,o,m}-Y^{(H)}_{i,o}$ over valid material--origin pairs, from
12-month drift, the material-average sum $\sum_{t=1}^{12}(\hat y_{i,t,m}-y_{i,t})$.
The separate absolute-error decomposition is
\begin{equation}\label{eq:decomp}
\mathrm{MAE}_m=\underbrace{\tfrac{1}{N}\!\sum_{t:\,y_t=0}\!|\hat y_t|}_{\text{zero-period}}
\;+\;\underbrace{\tfrac{1}{N}\!\sum_{t:\,y_t>0}\!|y_t-\hat y_t|}_{\text{non-zero-period}}.
\end{equation}
Here $N$ counts evaluated material--month pairs; the zero-period error share
divides the first term by total MAE. In the industrial diagnostic, 53--54\% of Croston/SBA error is the zero-period term
(over-prediction when nothing was demanded) versus $\approx$26\% for Naive:
positive errors on zero-demand months can offset under-prediction of demand
spikes. The 12-month drift diagnostic on a separate 500-material industrial
sample makes the directional split visible.
Methods with positive cumulative drift include Croston ($+8.9$, 86\% of
materials over-forecast), ADIDA ($+8.5$), Moirai ($+8.8$), SBA ($+7.3$), LightGBM
($+5.3$), TSB ($+4.5$). The accuracy-leading foundation models drift the
other way (TimesFM $-8.3$, only 51\% over-forecast; Chronos-Bolt-tiny $-7.1$), and
deliver lower service under the same policy. Accuracy and service thus favor
different bias directions here.

The public RAF and M5 panels carry no multi-item order log, so to compare the three
regimes we use the per-material fill proxy: the fraction of positive-demand
periods fully served, averaged equally over evaluated material windows.
A no-demand window scores 1. Partial demand is consumed and shortages are lost;
replenishment to $S$ is triggered when inventory position is at most
$s=\max(0,S-\hat\mu R)$. This proxy is available on all panels
(Fig.~\ref{fig:scatter}). The same inversion appears on the industrial panel ($\rho=-0.582$, 38
methods), weakens on RAF ($-0.173$, 38 methods; this replay includes the
intermittent specialist iMAPA in place of Prophet), and turns positive on
M5 ($+0.359$, 38 methods). The industrial per-material value is
comparable in magnitude to the contractual order-level inversion ($\rho=-0.555$,
above), while the cross-panel pattern provides
evidence of regime dependence (Section~\ref{sec:regime}). The RAF correlation
uses pooled rolling-origin accuracy from its separate service-replay stream,
not the per-material accuracy column of Table~\ref{tab:leaderboard}; method
records are paired, but available material coverage differs for some methods.

\begin{figure*}[t]
\centering
\includegraphics[width=0.92\textwidth]{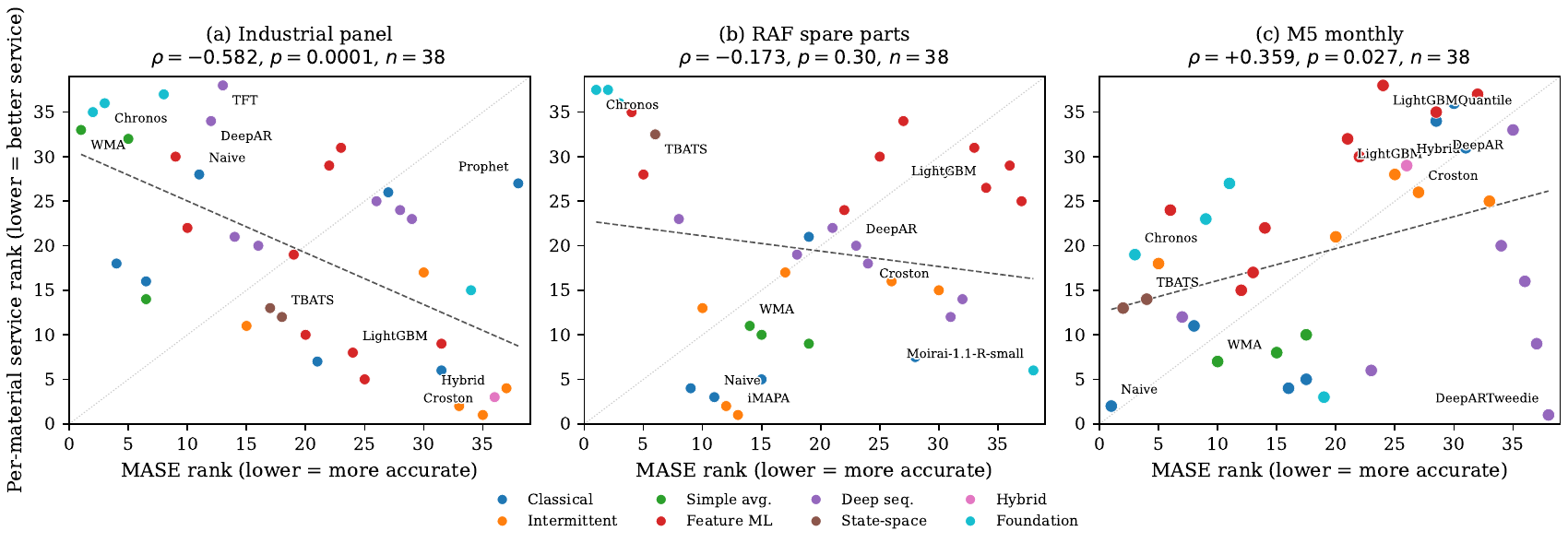}
\caption{MASE rank vs.\ per-material service rank at the 80\% target:
(a)~industrial ($\rho=-0.582$), (b)~RAF ($\rho=-0.173$), (c)~M5 ($\rho=+0.359$).}
\label{fig:scatter}
\end{figure*}

\noindent\textbf{Robustness to policy form.}
We first test whether the inversion depends on the Gaussian-$z$
order-up-to policy (Eq.~\ref{eq:basestock}). Rerunning the full 38-method
diagnostic on the per-material fill proxy under a \emph{distribution-aware}
policy (ordering to the empirical $\tau$-quantile of each method's
lead-time-demand distribution) \emph{strengthens} the inversion at every service
level. In a matched comparison at the 80\% target the per-material $\rho$ moves from
$-0.59$ (Gaussian) to $-0.70$ (distribution-aware, all $p<10^{-4}$), and the
contractual order-level inversion ($\rho=-0.555$) is of the same sign and magnitude.
Lower MASE is associated with worse service under both tested policy forms,
suggesting that the inversion extends beyond the Gaussian approximation.

\noindent\textbf{Scope: the common-policy regime.}
A further concern is that the over-forecasters merely run at a higher effective
service level, and that giving each method its own safety factor would erase the
inversion. We tested this on the full order book by re-deriving each method's
order-up-to level under its own \emph{cost-optimal} $z$, and under the more permissive
$z$ that \emph{maximizes} its own service. The inversion attenuates:
$\rho$ moves from $-0.555$ at common $z$ to $-0.278$ under each method's
cost-optimal $z$ (no longer significant at the 5\% level, $p=0.09$), and
with a service-maximizing $z$ for each method it remains $-0.384$ ($p=0.02$). Within the tested range of $z$,
methods with negatively biased forecasts attain lower
complete-order fill rates (TimesFM and WMA reach only COFR
$0.57$--$0.63$ at the largest safety factor, while the over-forecasters reach
$0.97$--$0.98$). Per-method retuning requires order-level simulation, whereas the initial BDD
screen uses forecast diagnostics. The inversion is clearest under the common-$z$ policy; cost-optimal tuning
attenuates it below statistical significance.

\subsection{Why the inversion is regime-specific}
\label{sec:regime}
The same diagnostic helps explain the cross-dataset pattern. On RAF
the bias populations roughly balance: foundation models negative ($-7.7$ to
$-7.2$), Croston-family positive ($+2$ to $+8$), and the rank correlation weakens
to $-0.173$. On M5, zeros are rare in the evaluation window ($\approx$5\% of
monthly cells in the final 12 months, versus 33\% over all 60 months in
Table~\ref{tab:data}), so the zero-period error
term appears insufficient to overturn the accuracy ordering. The accuracy--service relationship is positive ($+0.359$).
The inversion is consistent with an interaction between heavy intermittency and directional
cumulative forecast bias. Matched-subset checks suggest that sparsity alone is insufficient
(Section~\ref{sec:robustness}).

\subsection{Service efficiency and the 80-of-80 target}
\label{sec:frontier}
Within a single method, scaling the forecast trades inventory for service
(Section~\ref{sec:ladder}). Table~\ref{tab:frontier} shows the same trade
\emph{across} methods, against a representative \emph{80-of-80} target: the share of
the 41 platforms reaching COFR\,$\ge0.80$ when the 20{,}330 real multi-item orders are
replayed, met if $\ge80\%$ of platforms do. Three findings emerge. (i)~The lowest-MASE methods miss the target: the
leading foundation models hold the \emph{least} inventory ($0.37$--$0.54\times$ Naive) and rarely meet the platform threshold ($\le2\%$ of platforms), and the order-book leader WMA clears only 12\%.
(ii)~Only the over-forecasting methods clear it (Croston, SBA, ADIDA, HybridHierarchical,
TwoStageML, ETS; $80$--$93\%$ of platforms; Table~\ref{tab:frontier} shows
four of the six), at $1.2$--$1.5\times$ Naive's inventory:
directional bias improves service at a measurable inventory cost. (iii)~The target reproduces the
inversion: the most accurate methods miss it most severely. The platform roll-up
\emph{corroborates} the order-level reading rather than adding a criterion: the same six
methods clear both (Spearman 0.93 between the COFR and \%-platform columns),
indicating strong agreement between the rankings.
We \emph{report}, not optimize, this target (Section~\ref{sec:setting}).

\begin{table}[t]
\centering\footnotesize
\caption{Service--inventory frontier against the 80-of-80 target (industrial):
order-level COFR (Eq.~\ref{eq:cofr}), per-material inventory relative to Naive,
and the share of platforms at COFR$\,\ge\,$0.80.}
\label{tab:frontier}
\setlength{\tabcolsep}{3pt}
\begin{tabular}{lrrrc}
\toprule
Method & COFR & Inv.\ ($\times$Naive) & \%\,plat.\,$\ge0.80$ & 80-of-80 \\
\midrule
TimesFM            & 0.544 & 0.37 & 2\%  & $\times$ \\
Chronos-Bolt-small & 0.567 & 0.54 & 2\%  & $\times$ \\
WMA                & 0.558 & 0.96 & 12\% & $\times$ \\
Naive              & 0.627 & 1.00 & 15\% & $\times$ \\
Moirai             & 0.719 & 1.38 & 54\% & $\times$ \\
LightGBM           & 0.771 & 1.16 & 51\% & $\times$ \\
TSB                & 0.793 & 1.22 & 56\% & $\times$ \\
SBA                & 0.847 & 1.36 & 80\% & \checkmark \\
Croston            & 0.858 & 1.42 & 88\% & \checkmark \\
ADIDA              & 0.906 & 1.49 & 93\% & \checkmark \\
HybridHierarchical & 0.903 & 1.47 & 90\% & \checkmark \\
\bottomrule
\end{tabular}
\end{table}

\noindent\textbf{Operationalizing the BDD.}
The BDD (Fig.~\ref{fig:diagnostic-box}) screens forecasts before simulation:
flag low-MASE methods with negative cumulative drift as stockout risks and assess
the inventory cost of positive drift. Final COFR validation is required before deployment.

\begin{figure}[t]
\centering
\fbox{\begin{minipage}{0.93\linewidth}
\footnotesize
\textbf{Bias-Direction Diagnostic (BDD): deployment check}\\[1pt]
\emph{Input:} candidate forecasts; recent demand history.
\begin{enumerate}\setlength{\itemsep}{1pt}\setlength{\parskip}{0pt}
  \item Compute MASE and mean \emph{signed} lead-time bias.
  \item Compute the zero-period over-prediction share and the sign of cumulative drift.
  \item Flag: low MASE $+$ \emph{negative} drift $\Rightarrow$ stockout risk;
        \emph{positive} drift $\Rightarrow$ implicit safety stock; low MASE $+$
        near-zero drift $\Rightarrow$ deployment candidate.
  \item If negatively biased, correct the forecast \emph{mean} (non-zero re-centering, or
        scale toward bias-neutral $\alpha^\star$), with subsequent service and inventory validation.
  \item Validate on COFR at the contractual target before deployment.
\end{enumerate}
\end{minipage}}
\caption{A reusable forecast-side deployment check.}
\label{fig:diagnostic-box}
\end{figure}

\subsection{Correcting the bias inside the model}
\label{sec:adaptation}
The bias-direction mechanism motivates a correction \emph{within the model}.
Across a 31-configuration fine-tuning sweep (low-rank adaptation (LoRA)
ranks 4/8/16, data filtering, longer schedules), no configuration outperforms
zero-shot Chronos-2. Material-level 5-fold cross-validation gives the same result.
With about 70\% of observations equal to zero, fine-tuning may favor
near-zero predictions and reinforce the bias. These results motivate examining normalization as a possible source of bias.

Following the diagnostic in
Section~\ref{sec:diagnostic}, we examine how Chronos-2 normalizes each input series:
\begin{equation}\label{eq:in-std}
\hat x_t=\frac{x_t-\bar x}{s},\qquad \bar x=\mathrm{mean}(x),\;\; s=\mathrm{std}(x),
\end{equation}
Across 4,558 materials over the full 81-month descriptive window, the median
overall mean is 0.47 and the median positive-demand mean is 2.10. These are
separate panel summaries: the median within-material ratio is 5.06, and 80.3\%
of materials exceed $2\times$ (Fig.~\ref{fig:instancenorm}). Equation~\ref{eq:in-std}
therefore centers the representation on a mean dominated by zeros, which
may contribute to the negative bias observed in the benchmark. We replace it with the non-zero-conditional
mean $\bar x_{\mathrm{nz}}=\mathbb{E}[x\mid x>0]$ and standard deviation
$s_{\mathrm{nz}}=\mathrm{std}(x\mid x>0)$ (Eq.~\ref{eq:in-nz}); we refer to this normalization change as
\emph{non-zero-conditional re-centering}.
\begin{equation}\label{eq:in-nz}
\hat x^{\mathrm{nz}}_t=\frac{x_t-\bar x_{\mathrm{nz}}}{s_{\mathrm{nz}}}.
\end{equation}
Normalization uses each origin's training context only. With no positive
observations the center is zero; a zero conditional scale (including a single
positive observation) is replaced by the model's $\epsilon$. The stored center
and scale also invert the transform, after undoing any enabled arcsinh mapping.
The change increases forecast error (MASE $0.831\!\to\!0.948$, consistent with the pretrained
weights being adapted to the original centering) but corrects bias ($-2.52\!\to\!-0.10$) and
lifts the 90\% per-material fill proxy from 0.775 to \textbf{0.920} ($+14.5$pp), at
\emph{zero training cost}. On the complete-order service metric the same correction lifts
Chronos-2's COFR from 0.54 to \textbf{0.63} (platforms passing $5\%\to15\%$).
The separate hurdle run's mean lead-time forecast rises from 6.5 to 8.9
(Table~\ref{tab:hurdle}), a forecast ratio of $\approx1.4\times$, not a measured stock ratio.
This brings order service close to the Naive baseline.

The corrected model still misses the 80-of-80 target that specialists meet with
$1.4$--$1.5\times$ Naive's inventory. The proxy gain holds across targets ($+14$ to
$+17$pp) and \emph{replicates on RAF} ($+17.1$pp, $0.607\!\to\!0.778$; the centering discrepancy
is widespread there). Table~\ref{tab:adaptation} compares selected approaches from the
thirteen strategies (over 140 configurations) we tested (its service column
is the per-material fill proxy, not the order-level COFR of Eq.~\ref{eq:cofr}):
MixFT routes materials to cluster-specific LoRA adapters and approximately
matches zero-shot accuracy. Pruning 20\% increases MASE by approximately $0.0024$, suggesting some parameter redundancy
for this task. Changes to normalization, forecast extraction, and downstream
calibration improve service more than fine-tuning in these tests.

\begin{table}[t]
\centering\footnotesize
\caption{Chronos-2 adaptation: pooled MASE (4,558 materials; origins 69, 75),
$\Delta$MASE vs zero-shot (positive = better; computed before rounding), and
fill proxy at target 0.90. SBC: Syntetos--Boylan class-balanced training.}
\label{tab:adaptation}
\setlength{\tabcolsep}{4pt}
\renewcommand{\arraystretch}{1.2}
\begin{tabular}{lrrr}
\toprule
Strategy & MASE & $\Delta$MASE & Fill proxy@0.90 \\
\midrule
Zero-shot (baseline)                          & 0.831 & n/a      & 0.775 \\
InstanceNorm $\bar x_{\rm nz}$ (no training)  & 0.948 & $-0.117$ & \textbf{0.920} \\
MixFT sub-domain (K=6)                        & 0.832 & $-0.001$ & 0.753 \\
Magnitude pruning (20\%)                      & 0.834 & $-0.002$ & 0.776 \\
Quantile selection (90th pctl)                & 3.603 & $-2.771$ & \textbf{0.946} \\
Best LoRA (SBC-balanced r8)                   & 0.834 & $-0.003$ & n/a \\
\bottomrule
\end{tabular}
\end{table}

\begin{figure}[t]
\centering
\includegraphics[width=0.98\linewidth]{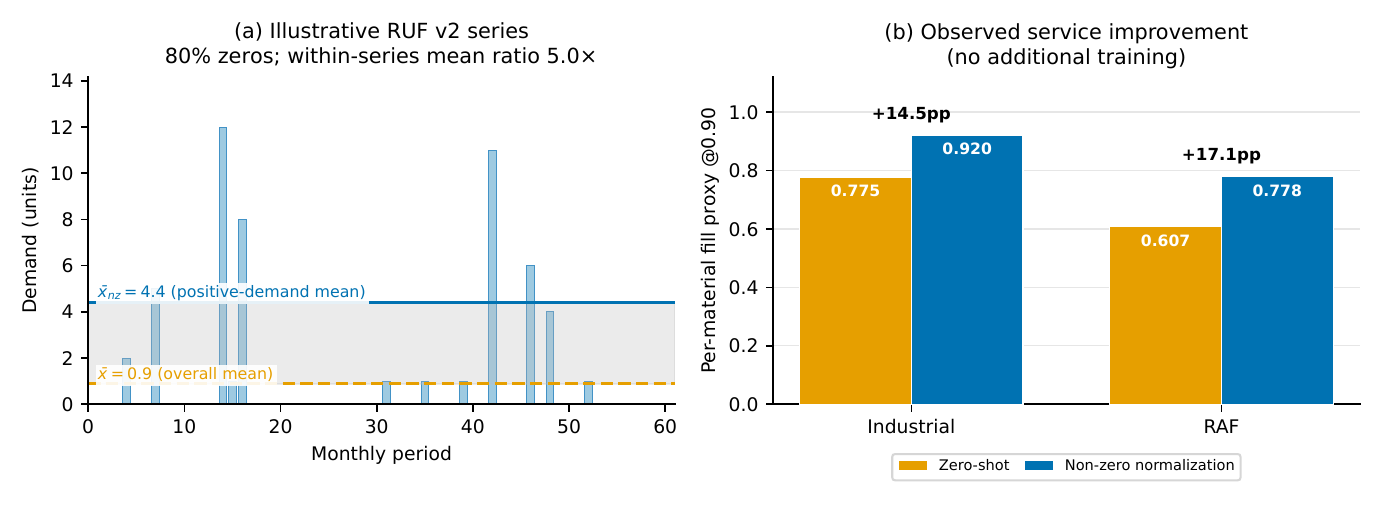}
\caption{The instance-normalization correction: (a)~an illustrative RUF v2
series showing the overall mean $\bar x$ and positive-demand mean $\bar x_{\mathrm{nz}}$;
(b)~the fill-proxy lift from re-centering on $\bar x_{\mathrm{nz}}$, both panels.}
\label{fig:instancenorm}
\end{figure}

\subsection{A continuous bias--service mechanism}
\label{sec:ladder}
We scale TimesFM's rate forecast as $\hat y'=\alpha\hat y$, holding dispersion
fixed, on 500 industrial materials at policy target 0.80. Training uses 69 months;
accuracy uses the next six months and the partial-fulfillment simulation spans
12 months ($L=6$, $R=1$). This is a \emph{per-material fill-proxy} study,
distinct from Table~\ref{tab:frontier}'s complete-order replay.
The unscaled reference ($\alpha=1$) has MASE 0.945 and fill proxy 0.636.
Across the tested $\alpha=0.5$--2.0 grid, proxy service increases from 0.52 to 0.72,
while mean inventory value grows from $0.3\times$ to $3.1\times$ its value at
$\alpha=1$. The most accurate setting ($\alpha=0.75$, MASE 0.859) serves less well
(proxy 0.585). The near-bias-neutral $\alpha^\star\approx1.39$ yields proxy 0.677
at $1.76\times$ unscaled inventory. This factor is calibrated retrospectively
from evaluation demand; prospective deployment requires a separate calibration
window. These are observed trade-offs, not a pathwise monotonicity guarantee.
Thus accuracy and service optima diverge within one method. Like scaling toward $\alpha^\star$, re-centering
reduces negative bias, but changes the representation without retraining.

\subsection{A structural (hurdle) test}
\label{sec:hurdle}
Because changing the forecast mean changes its stock target
(Proposition~\ref{prop:bias}), we also tested a \emph{structural} alternative: a
hurdle model that estimates occurrence $p_t=\Pr(Y_t>0)$ from the full history and size
$\mathbb{E}[Y_t\mid Y_t>0]$ from the non-zero-centered Chronos, then mixes them.
This is not a matched-mean experiment. The expected hurdle lowers mean LTD from
6.5 to 5.4 and increases fill proxy from 0.778 to 0.791, but trails direct
re-centering (0.924). The $\tau=0.9$ size quantile achieves 0.881 while raising
mean LTD to 18.2, bias to $+9.5$, and MASE to 2.20. The results do not isolate
mean, dispersion, and distributional shape. Table~\ref{tab:hurdle} uses a
separate full-panel run; its zero-shot and re-centering MASE and fill-proxy values
differ from Table~\ref{tab:adaptation} by less than 0.01.

A \emph{distribution-aware} policy that consumes the hurdle's full mixed quantile
also incurs an inventory penalty. The hurdle has lower service efficiency in the reported comparison:
for low-occurrence materials, its service quantile often
falls inside the zero mass, setting an order-up-to level of zero
(60.8\% of materials at $\tau{=}0.70$) and limiting their service. At matched fill proxy $0.85$,
it needs about $1.6\times$ the inventory of plain re-centering. Using the full forecast distribution raises
\emph{achievable} service at an inventory cost; mean-correcting re-centering is the most
inventory-efficient correction tested at the reported service targets.

\begin{table}[t]
\centering\footnotesize
\caption{Hurdle test: pooled MASE and service over all 4,558 industrial
materials at origins 69, 72, 75 (six-month horizons).
LTD: lead-time demand; service is the per-material fill proxy.}
\label{tab:hurdle}
\setlength{\tabcolsep}{4pt}
\begin{tabular}{lrrrr}
\toprule
Method & MASE & Bias & Mean LTD & \shortstack{Fill proxy\\@0.90} \\
\midrule
Zero-shot (standard norm)   & 0.834 & $-2.27$ & 6.5  & 0.778 \\
Non-zero re-centering       & 0.955 & $+0.14$ & 8.9  & \textbf{0.924} \\
Hurdle, expected            & 0.926 & $-3.31$ & 5.4  & 0.791 \\
Hurdle, $\tau{=}0.7$ size    & 1.054 & $-1.25$ & 7.5  & 0.771 \\
Hurdle, $\tau{=}0.9$ size    & 2.198 & $+9.52$ & 18.2 & 0.881 \\
\bottomrule
\end{tabular}
\end{table}

\subsection{The findings reproduce on the RUF panel}
\label{sec:repro}
On the released RUF panel (Section~\ref{sec:benchmark-release}) the
main findings extend to the full 38-method roster. The accuracy leaders again include foundation models and the weighted
moving average (WMA, TimesFM, and Chronos-Bolt occupy MASE ranks 1--4). The
decision-facing inversion reproduces across the 5{,}000-material panel: MASE rank
and per-material service rank correlate at $\rho=-0.47$ at the 80\% target, negative at
every service level (run on the released RUF v2 panel; significance, bootstrap,
and regeneration-seed checks are in Table~\ref{tab:robust}). This is the same sign
as the real panel and comparable to the
industrial per-material $-0.58$. The training-free instance-normalization correction
(Section~\ref{sec:adaptation}) reproduces too: re-centering Chronos-2 on
non-zero-conditional statistics lifts the per-material fill proxy by $+17.6$ to
$+21.4$pp across service targets ($79.1\%$ of materials show the $>2\times$ centering
artifact), with an increase in MASE ($0.765\!\to\!0.807$), the same trade-off as on the
industrial panel. The public panel corroborates the leaderboard, correction, and inversion
direction; contractual magnitude remains a private-panel result.

\subsection{Robustness}
\label{sec:robustness}
We assess the sensitivity of the main results
to the checks in Table~\ref{tab:robust}, including the attenuation under per-method policy tuning.

\begin{table}[t]
\centering\footnotesize
\caption{Threats to validity and corresponding robustness checks. The policy-form check is summarized in
Section~\ref{sec:diagnostic}.}
\label{tab:robust}
\setlength{\tabcolsep}{3pt}
\renewcommand{\arraystretch}{1.15}
\begin{tabularx}{\linewidth}{@{}>{\raggedright\arraybackslash}p{0.16\linewidth} >{\raggedright\arraybackslash}p{0.23\linewidth} >{\raggedright\arraybackslash}p{0.32\linewidth} >{\raggedright\arraybackslash}X@{}}
\toprule
Threat & Check & Result & Reading \\
\midrule
Metric artifact & MASE lag-1/3/12 & rankings stable, $\rho\ge0.99$ & stable across tested scales \\ \hline
Lucky sample & caps$\times$seeds (12 cfg) & top cluster identical, $\rho\ge0.988$ & stable across tested samples \\ \hline
Sparsity alone & RAF-matched industrial subsets & inversion persists, $\rho\,{\approx}\,{-}0.42$ to ${-}0.47$ & sparsity alone insufficient \\ \hline
Gaussian-policy artifact & distribution-aware quantile policy (\S\ref{sec:diagnostic}) & inversion \emph{strengthens}, $-0.59\!\to\!-0.70$ & holds under both forms \\ \hline
Lead-time artifact & order-level $\rho$ at $H{\in}\{1,3,6,12\}$ & sign stable, \emph{strengthens} with $H$, $-0.47$ to $-0.61$ & stable across tested horizons \\ \hline
Per-method $z$-tuning & cost-optimal / service-max $z$ per method & sign remains negative; cost-optimal $z$: $\rho=-0.28$, $p=0.09$; service-max $z$: $-0.38$, $p=0.02$ & attenuated under per-method tuning \\ \hline
``Just overstocking'' & bias ladder (\S\ref{sec:ladder}) & service costs inventory; bias-neutral operating point & a controlled trade \\ \hline
Adaptation overfit & material-level 5-fold CV & zero-shot 0.905 vs.\ LoRA $\ge$0.918 & holds out-of-sample \\ \hline
Reliance on private data & synthetic 38/38 replay (5{,}000 mat.) & reproduces, $\rho=-0.47$ ($p<0.05$ at every target; $p<0.005$ for 0.70--0.90, $n=38$); CI $[-0.76,-0.10]$ excludes 0; seeds 43--47: $-0.41$ to $-0.61$; gate PASS & publicly checkable \\
\bottomrule
\end{tabularx}
\end{table}

All cited streams are \emph{audited} (method failures are logged and counted, never
silently replaced), and the panels are kept distinct (the industrial slices are
enumerated in the Section~\ref{sec:data} footnote). RAF uses a seed-fixed,
Syntetos--Boylan-stratified 500-series sample, sized by the cost of running 38
methods over rolling origins. Table~\ref{tab:leaderboard} ranks
full-coverage audited methods; failure-subset scores are unranked. Its BATS/TBATS
scores use 500 series, whereas the separate rolling-origin service stream uses
100 for their state-space refits. Thus RAF's 38-method correlation is an
available-case replication, not a uniformly matched-material comparison.
The caps$\times$seeds check (Table~\ref{tab:robust}) leaves the top
cluster identical ($\rho\ge0.988$). No tested protocol variation reverses the sign of the industrial inversion, although
its magnitude and statistical significance depend on the evaluation choices.

\section{Conclusion}
\label{sec:conclusion}
We benchmarked 38 forecasting methods across eight model classes on an
industrial spare-parts panel and public panels. The accuracy ranking
\emph{inverts} against the contractual complete-order service target on the
industrial panel. The corresponding relationship with per-material service
weakens on the public RAF panel and is positive on low-intermittency M5: a regime-dependent
forecast--inventory relationship whose sign tracks forecast bias direction.
Examining that bias in a foundation model's instance normalization yielded a
zero-training-cost correction that recovers lost service on both panels,
including order-level COFR on the industrial panel.

For deployment, accuracy and service can favor different model choices: WMA
tops the order-relevant industrial accuracy ranking, while bias direction helps
explain service differences. Correcting this bias improves service despite increasing forecast error.
For Chronos-2, modifying normalization is effective, while the tested fine-tuning
configurations do not improve on zero-shot accuracy.

\emph{RUF} addresses the reproducibility barrier posed by confidential industrial data.
We rerun the findings on a regenerate-until-fidelity surrogate certified against the real panel.
The leaderboard, correction, and inversion direction reproduce. Adapted to other count or continuous
series, the generate-then-certify loop could provide open certified surrogates
for confidential corpora and decision-aware benchmarks in other domains.

\textbf{Limitations and future work.} The industrial evidence comes from one domain at monthly
granularity, and the 80-of-80 target is illustrative. The M5 comparison suggests that the inversion
depends on heavy intermittency with directional forecast bias; we do not claim the negative
correlation transfers to fast-moving demand. The diagnostic holds a common inventory policy
fixed to isolate forecast effects; future work should jointly optimize forecast
choice and cost-aware inventory policy, including distribution-aware decisions and copula dependence. The service gains from
distribution-aware decisions carry an inventory cost
(Section~\ref{sec:hurdle}). A second direction is directive-conditioned forecasting,
which incorporates forward-looking operational information unavailable from demand
history alone. Finally, the reported service gains carry inventory trade-offs;
Proposition~\ref{prop:bias} orders stock targets, not realized service. Jointly optimizing complete-order service and inventory cost is therefore a
multi-objective next step.

Practitioners should measure forecast bias direction alongside accuracy before
deployment. Under an asymmetric service contract, the lowest-error forecast
need not deliver the highest service.

\noindent\textbf{Availability.} The RUF benchmark (v1 structural and v2
parametric, panel of Section~\ref{sec:repro}), its generators,
evaluation code, full public-panel leaderboard, and a per-artifact reproduction
guide are at \url{https://github.com/sfcheng-research/icdm-2026-reproduction}. The
industrial panel is confidential.


\end{document}